\documentclass{article}
\usepackage[utf8]{inputenc}
\usepackage{microtype}
\usepackage{graphicx}
\usepackage{subfigure}
\usepackage{booktabs} 
\usepackage{hyperref}
\usepackage{url}
\usepackage{amsthm}
\usepackage{amsmath}
\usepackage{braket}
\usepackage{amsfonts}
\usepackage{mathrsfs}
\usepackage{graphicx}
\usepackage[ruled,vlined,linesnumbered]{algorithm2e}
\usepackage{multirow}
\usepackage[center]{caption}
\graphicspath{ {./} }

\newtheorem{thm}{Theorem}

\newtheorem{lem}{Lemma}

\title{Learning High Order Feature Interactions with Fine Control Kernels}
\author{Hristo Paskov \\ hpaskov@alumni.stanford.edu
\and Alex Paskov \\ asp2205@columbia.edu
\and Robert West \\ robert.west@epfl.ch}
\date{}

\begin{document}
\maketitle

\begin{abstract}
We provide a methodology for learning sparse statistical models that use as features all possible multiplicative interactions among an underlying atomic set of features. While the resulting optimization problems are exponentially sized, our methodology leads to algorithms that can often solve these problems exactly or provide approximate solutions based on combining highly correlated features. We also introduce an algorithmic paradigm, the Fine Control Kernel framework, so named because it is based on Fenchel Duality and is reminiscent of kernel methods. Its theory is tailored to large sparse learning problems, and it leads to efficient feature screening rules for interactions. These rules are inspired by the Apriori algorithm for market basket analysis -- which also falls under the purview of Fine Control Kernels, and can be applied to a plurality of learning problems including the Lasso and sparse matrix estimation. Experiments on biomedical datasets demonstrate the efficacy of our methodology in deriving algorithms that efficiently produce interactions models which achieve state-of-the-art accuracy and are interpretable.
\end{abstract}

\section{Introduction}
Machine learning algorithms that can uncover and exploit nonlinear multivariate relationships are of fundamental importance to science, medicine, technology, and business. To highlight this need, consider the behavior of a model restricted to learn arbitrary \emph{univariate} functions of its inputs which are then added together to form an output, i.e. a generalized additive regression model. Such a model cannot ``discover" Newton's second law of motion $F=ma$ because force involves a multiplicative interaction between mass and acceleration. Even worse, the model's bias can mislead it to find \emph{non-monotonic functions} of acceleration or mass that actually capture peculiarities in the distribution of the training data. The astute reader may argue that our model works perfectly well if used to predict $\log F$ (provided all quantities are strictly positive) to convert its group operation from addition to multiplication. This remedy is unfortunately limited; it fails under the demands of complex real-world phenomena that require sums of multivariate products (or other nonlinear transformations) to model correctly.

At the same time, finding interactions -- even simple multiplicative ones -- is a challenging problem; an exhaustive search must check $\Theta(2^d)$ combinations when searching for interactions among $d$ features. As such, scalable methods for nonlinear multivariate modelling that explicitly find interactions do so only approximately. For example, decision trees are typically built greedily top-down or bottom-up \cite{ESL}. This greediness can diminish performance and miss important splits \cite{bertsimas2017optimal}, so it is typically necessary to train ensembles of trees \cite{breiman2001random} to, at least partly, ameliorate this issue. However, ensemble methods further increase training times and complicate model interpretability. 

Kernel methods \cite{kernels}, particularly ones based on ANOVA, RBF, or polynomial kernels, form another important class of techniques for learning interactions. These methods are top performers on a variety of learning tasks, but they require computing a Kernel matrix which can be prohibitively expensive on large training sets. Furthermore, they learn \emph{implicit} representations and are thereby not easy to interpret and not compatible with a plurality of modern sparsity inducing regularizers such as the Lasso.



This work investigates scalable methods for training convex "loss plus sparse regularizer" models
\begin{equation}\label{eq:main_primal}
    \underset{\beta }{\text{min }} \ell\left(\mathscr{P}\left(A\right)\beta\right) + r\left(\beta\right).
\end{equation}
that directly learn with $\mathscr{P}\left(A\right)$, a matrix of all $2^d -1$ possible element-wise products of columns from an ``atomic" feature matrix $A \in [0,1]^{n \times d}$. Here $r$ is a sparsity inducing regularizer such as the $\ell_1$ norm; we heavily rely on sparsity to ensure that the resulting model has a computationally tractable number of non-zero coefficients. We also discuss learning scenarios that lead to intractable problems and provide simple approximations based on combining highly correlated features in $\mathscr{P}(A)$.

We provide two main contributions en route to addressing the exponentially-sized optimization problem specified by equation (\ref{eq:main_primal}). The first is an algorithmic paradigm that provides design patterns for tractable optimization in high dimensional model spaces where explicitly storing a dense coefficient vector may be impossible. We term this paradigm the Fine Control Kernel framework because it is based on duality -- in the same manner as the kernel trick, but it is compatible with a wide range of regularizers and is specifically geared towards sparsity. It allows the feature matrix to be treated as a nonlinear operator whose nonlinear structure is dictated by the regularization, and to compute, but never store, primal solutions on the fly. Section \ref{sec:fine_control} outlines the basic theorems, their intended usage, and the optimization techniques that are central to Fine Control Kernels.

Our second main contribution shows how to tractably multiply with $\mathscr{P}(A)$ in the context of the Fine Control Kernel paradigm. Theorem \ref{thm:any_vector} allows us to determine which columns of $\mathscr{P}(A)$ pertain to nonzero coefficients in the model without having to inspect these columns. This theorem provides conditions under which an interaction involving features $u \subset \{1,\dots,d\}$ certifies that all higher order interactions (involving $u$ and other features) may be ignored. This reasoning provides a methodology for deriving ``monotonic" feature screening rules that can be used with a variety of sparsity inducing regularizers to intelligently prune the columns of $\mathscr{P}(A)$ to a tractably sized set. It is a generalization of the \emph{downward closure} property used by the Apriori algorithm \cite{market_basket} for market basket analysis. Section \ref{sec:interactions} discusses the theorem, provides a simple algorithm for the Lasso, and discusses ways to combine overly correlated columns of $\mathscr{P}(A)$ to allow for tractable multiplication when the sparse regularization level is too low to sufficiently prune features.

We demonstrate the viability of our approach in section \ref{sec:experiments} by showcasing three different learning objectives that use $\mathscr{P}(A)$ as a feature matrix and are solved using the Fine Control Kernel paradigm. The first is a ``single-class'' problem in which we wish to find suspiciously frequent drug co-occurrences among nearly 9 million records in the FDA's Adverse Event Reporting System \cite{faers} to propose potential drug interactions. The other experiments are supervised learning problems pertaining to HIV drug screening and quantum chemistry tasks in MolecueNet \cite{moleculenet}. We solve an Elastic-Net regularized logistic regression model for the former and a nuclear norm and group Lasso regularized matrix learning problem for the latter. The interaction models provide substantial accuracy improvements over our baselines in all cases. Moreover, all of these experiments required, at most, on the order of hours to compute entire regularization paths on a 2017 MacBook Pro latop with $16$ gigabytes of RAM.

\subsection{Related Work}
A plurality of approximate algorithms have been developed over the years for learning with interactions \cite{ESL}. Beyond the aforementioned methods, multiadaptive regression splines \cite{friedman1991multivariate} feature prominently in statistics. There has recently been renewed interest in interactions, with a variety of publications aiming to develop convex regularizers -- and corresponding scalable algorithms, that are statistically well suited to this setting. Notable examples include \cite{tibshirani2019pliable, bien2013lasso, yu2019reluctant, hazimeh2019learning}. The XYZ algorithm in \cite{thanei2018xyz} takes a stochastic approach to identifying interactions. All of this work focuses specifically on \emph{pairwise} interactions, which are less computationally challenging than the set of all possible interactions.

Another related research thread is that of convex neural networks \cite{bengio2006convex} which aim to solve a convex objective over an infinite dimensional feature representation. Recent work based on the conditional gradient method \cite{boyd2017saturating, rakesh} shows that good predictive performance can be achieved, but the quality of the solution heavily determines predictive accuracy. These methods can require large amounts of computation to provide high quality solutions and are based on heuristics. 

The work on convex neural networks presents a fascinating connection between the convex objectives typically seen in statistics and the rich feature representations from deep learning. The work of these two communities begins to converge when convex learning problems are equipped with sufficiently rich feature representations. The Fine Control Kernel paradigm aims to solve these types of problems explicitly whenever the feature representation and sparse regularization have sufficient algorithmic structure. At an abstract level, where current approaches to learning isolate the feature matrix as a linear operator (typically an explicit matrix), our work explores what is possible when the combination of rich feature representation and sparse regularization are treated together as a \emph{nonlinear} operator. 

\section{Fine Control Kernels}\label{sec:fine_control}
The entry point for our theory is a refinement of equation (\ref{eq:primal}) that splits the regularization penalty into three terms, each elucidating a different aspect of the learning problem's duality or combinatorial structure. The primal form of the learning problems we are interested in may be written as
\begin{equation}\label{eq:primal}
    \underset{\beta \in \mathbb{R}^d}{\text{min }} \ell\left(X\beta\right) +\left[ \sum_{\mathcal{C} \in \Gamma} \lambda_{\mathcal{C}} \; \underset{s \in \mathcal{C}}{\text{max }} s^\top\beta + \phi\left(\beta\right)+ \omega\left(\beta\right)\right].
\end{equation}
We have replaced the interactions matrix with an arbitrary feature matrix $X$. The first regularization term determines the sparsity structure of the learning problem. It sums over $\Gamma$, a collection of convex sets, each of which must contain $\mathbf{0}$. This decomposition is not unique, for example, 
\[|x|+|y| = \underset{s \in [-1,1]^2}{\text{max}} s_1x + s_2y = \underset{s \in [-1,1]}{\text{max}}sx + \underset{s \in [-1,1]}{\text{max}}sy. \]
It will often be most useful to define $\Gamma$ to have the finest granularity possible so each convex set does not contain Cartesian products.
Functions of the form $f(x)=\underset{s \in \mathcal{C}}{\text{max }} s^\top x$, where $\mathcal{C} \subset \mathbb{R}^d$  is convex, are known as \emph{support functions} in convex analysis \cite{convex}. They capture a variety of convex functions that are not differentiable at $\mathbf{0}$ and can therefore set to zero linear transformations of coefficients in a nontrivial way. A simple example is given by the Lasso \cite{tibshirani1996regression}, whose associated $\Gamma$-collection is the collection of all axis-aligned line segments
\[\Gamma = \left\{\left\{s \mathbf{e}_1 \mid -1 \leq s \leq 1\right\}, \dots,\left\{s \mathbf{e}_d \mid -1 \leq s \leq 1\right\}\right\}.\] 
Here $\mathbf{e}_k$ is the $k^{\text{th}}$ canonical basis vector comprised of all zeros except for a single $1$ at position $k$. Other important examples include group-wise regularizers, the generalized Lasso, and the nuclear norm \cite{sparse_learning}. 


The second regularization term links the solutions of the primal and dual problems. Function $\phi$ plays a role akin to the strictly convex function used in the Bregman divergence for algorithms such mirror descent \cite{first_order}, but we do not require strict convexity. Instead, $\phi$ must have a Fenchel conjugate that is easy to work with and which does not alter the sparsity structure induced by the support functions involving $\Gamma$. Formally, $\phi$ is \emph{compatible} with $\Gamma$ if 
\[\Pi_{\mathcal{C}}v = \mathbf{0} \Rightarrow \Pi_{\mathcal{C}}u = \mathbf{0}, \; \forall v \in \partial \phi(u), \; \forall \mathcal{C} \in \Gamma,\]
where $\Pi_{\mathcal{C}}$ projects onto the subspace containing $\mathcal{C}$.
The interpretation of this condition will become clear momentarily, and it is worthwhile to note that the constant function $\phi(\beta) = 0$ and simple quadratic $\|\beta\|_2^2$ are valid ``universal" choices that always satisfy this condition. Finally, the $\omega$-term captures any priors on $\beta$ that are not sparsity inducing (such as bound constraints) that are not included in $\phi$.

We are now ready to state our main duality theorem which uses \emph{Fenchel conjugation} \cite{convex}, $f^*(x^*) = \underset{x}{\text{sup }} x^\top x^* - f(x)$, to provide a dual problem for equation (\ref{eq:primal}). It is a restatement of Fenchel Duality \cite{convex} specifically tailored to the structure of sparse learning problems, and it automates the computation of dual problems given a table of Fenchel conjugates as in Appendix \ref{sec:app_conj} or any standard text on convex optimization \cite{convex}. Its proof is in Appendix \ref{sec:app_thms}.
\begin{thm}\label{thm:primal_dual}
Suppose that the objective in equation (\ref{eq:primal}) is strictly feasible, bounded from below, and $\ell \circ X, \phi,\omega$ are closed proper convex functions. Define
\[\psi(\zeta)=\underset{\substack{s_{\mathcal{C}} \in \mathcal{C}, \forall \mathcal{C} \in \Gamma \\ \xi \in \mathbb{R}^d}}{\textnormal{min }}\left[\phi^*\left(\zeta - \sum_{\mathcal{C} \in \Gamma} \lambda_{\mathcal{C}} s_{\mathcal{C}} - \xi\right) + \omega^*(\xi) \right].\]
Then
\begin{equation}\label{eq:dual}
\begin{aligned}
    &\underset{\alpha \in \mathbb{R}^n}{\textnormal{max }} -\ell^*(-\alpha) - \psi\left(X^\top \alpha\right)
\end{aligned}
\end{equation}
is a dual problem. Any dual feasible $\alpha$ determines a candidate set of primal coefficients $\beta(\alpha)$ as the subdifferential
\begin{equation}\label{eq:beta_form}
    \beta(\alpha) = \partial \phi^*\left(X^\top\alpha - \sum_{\mathcal{C} \in \Gamma} \lambda_{\mathcal{C}}s_{\mathcal{C}} - \xi\right)
\end{equation}
where the $s_{\mathcal{C}}$ and $\xi$ are optimal for the minimization inside $\psi$.
At least one optimal primal solution exists, and $\beta,\alpha$ form an optimal primal-dual pair iff. $\beta \in \beta\left(\alpha\right)$ and $-\alpha \in \partial \ell(X\beta)$.
\end{thm}
We have chosen this particular dual form because it emphasizes the structure and operations that are prevalent in sparse learning problems. Function $\psi$ represents the regularization in the dual objective, and much of the algorithmic structure in the learning problem is determined by the computation of $\xi$, the $s_{\mathcal{C}}$ variables, and the subdifferential operator $\partial\phi^*$. Luckily, common regularizers exhibit favorable structure, such as separability, which can either allow us to easily compute all relevant quantities and operations on the fly or whose result requires comparably little memory to store.

For example, consider a bound constrained Elastic-Net
\[\underset{\beta \in \mathbb{R}^d}{\text{min }}\ell(X\beta) + \lambda \|\beta\|_1 +\frac{\tau}{2}\|\beta\|_2^2 \;\;\; \text{subject to }\;\mathbf{l} \preceq \beta \preceq \mathbf{u}.\]
Here $\mathbf{l},\mathbf{u} \in \mathbb{R}^d$ provide component-wise lower and upper bounds, respectively, for $\beta$. Applying our theorem yields
\[\underset{\substack{-1 \leq s_k \leq 1 \\ u_k,l_k \geq 0}}{\text{min }}\frac{1}{2\tau}\left\|\zeta - \sum_{k=1}^d \left(\lambda s_k + u_k - l_k\right) \mathbf{e}_k\right\|_2^2 - \mathbf{l}^\top l + \mathbf{u}^\top u\]
for $\psi$. Since the quadratic function $\|x\|_2^2$ is differentiable and strongly convex, $\beta(\alpha) = \left\{w(\alpha)\right\}$ is singleton with
\[w(\alpha)_k = \text{max}\left(\text{min}\left(S\left(X_k^\top \alpha,\lambda\right)/\tau, \mathbf{u}_k\right), \mathbf{l}_k\right),\]
where $S(x, \lambda) = \text{sign}(x)\text{max}(|x|-\lambda, 0)$ is the soft-thresholding operator. In this particular case it is easy to compute \emph{individual} primal coefficients because all priors in the primal objective are separable. This provides a time-memory tradeoff that allows us to compute, but never store, the primal coefficients. In particular, if $w(\alpha)$ contains many non-zero entries and $\alpha$ is relatively small, it may be prudent to store $\alpha$ instead of $w(\alpha)$ and predict on the fly via
\[w(\alpha)^\top x = \sum_{k}\text{max}\left(\text{min}\left(S\left(X_k^\top \alpha,\lambda\right)/\tau, \mathbf{u}_k\right), \mathbf{l}_k\right)x_k.\]
This example also demonstrates a connection to kernel methods which exploit the kernel trick to train a dual problem and avoid computing primal coefficients altogether, even for predictions. Consider what happens if we solely employ $\ell_2$ regularization in our example; application of Theorem \ref{thm:primal_dual} with $\Gamma = \{\emptyset\}$ and $\omega(x) = 0$ yields primal-dual problems 
\[\underset{\gamma \in \mathbb{R}^d}{\text{min }}\ell(X\gamma) + \frac{\tau}{2}\|\gamma\|_2^2 \Leftrightarrow \underset{\theta \in \mathbb{R}^n}{\text{max }}-\ell^*(-\theta) - \frac{1}{2\tau}\theta^\top XX^\top \theta.\]
The only dependence on $\mathbb{R}^d$ occurs in the computation of the $n \times n$ Gram matrix $K = XX^\top$ which the kernel trick \cite{kernels} prescribes computing as $K_{i,j} = k(\mathbf{x}_i, \mathbf{x}_j)$ for some kernel function $k: \mathcal{X}\times \mathcal{X} \rightarrow \mathbb{R}$. This introduces a level of indirection whereby the points $\mathbf{x}_i \in \mathcal{X}$ need not be in $\mathbb{R}^d$; instead they \emph{map} to $\mathbb{R}^d$ via the kernel function's feature map $\phi:\mathcal{X} \rightarrow \mathbb{R}^d$. Thus, the rows of $X$ actually store $\phi(\mathbf{x}_1),\dots,\phi(\mathbf{x}_n)$. Working with $K$ rather than $X$ removes all dependence in the dual problem on $\mathbb{R}^d$.

Predictions in this case are also amenable to the kernel trick. The set $\gamma(\theta) = \{X^\top \theta\}$ is again singleton, so predictions on a point $\mathbf{z} \in \mathcal{X}$ may be computed as
\vspace{-.5em}
\[\frac{1}{\tau}\theta^\top X \phi(\mathbf{z}) = \frac{1}{\tau}\sum_{i=1}^n\theta_i k(\mathbf{z}, \mathbf{x}_i).\]
The kernel trick works in this case because the primal coefficients are always in the row space of $X$. Contrast this with the prediction equation for the bounded Elastic-Net. A primal coefficient there also starts as a point in the row space of $X$, namely $X^\top \alpha$. However, this point is then translated by $\xi$ and the $s_{\mathcal{C}}$ and the result is passed through the subdifferential operator $\partial\phi^*$. These last two operations can push $\beta(\alpha)$ to lie outside of the row space of $X$. Thus, the price of the kernel trick's ability to remove a dual problem's dependence on its primal search space is its limitation to regularizers that ensure $\beta(\alpha)$ lies in the row space of $X$. More general regularization introduces a dependence on the primal search space since $\beta(\alpha)$ becomes a subset of this row space that is then perturbed to lie outside of it.

\subsection{Sparsity Structure}
Sparsity is a powerful way to mitigate the dual problem's dependence on the primal search space when using sophisticated regularization. We define the \emph{sparsity pattern} of a candidate solution $\beta$ for equation (\ref{eq:primal}) to be the subset $\Omega \subset \Gamma$ where $\Pi_{\mathcal{C}}\beta = \mathbf{0}$ for all $\mathcal{C} \in \Omega$. Determining a global minimizer's sparsity pattern is a major part of the optimization when $\Gamma$ is nontrivial. Oftentimes the learning problem simplifies considerably if this sparsity pattern, or a subset thereof, can be guessed ahead of time. Sparsity is also important en route to a minimizer. This is true even in dual formulations as in equation (\ref{eq:dual}). Computing a subgradient of the dual objective requires computing a subgradient of $\psi$, some $\beta \in \beta(\alpha)$ at the current dual variable iterate $\alpha$, and then computing the product $X^\top \beta$. This subgradient and matrix-vector product are impossible to compute na\"{i}vely when the  primal search space is intractably large; in the case of interactions this space is exponentially sized. 

The theorem below, whose proof we relegate to Appendix \ref{sec:app_thms}, uses compatibility to show how dual coefficients $\alpha$ determine the sparsity pattern of their corresponding primal coefficients $\beta \in\beta(\alpha)$. It provides a simple condition that determines when $\Pi_{\mathcal{C}}\beta = \mathbf{0}$ for the $\mathcal{C}\in\Gamma$.
\begin{thm}\label{thm:dual_zeros}
Fix $\alpha \in \mathbb{R}^n$, $\mathcal{V} \in \Gamma$ with $\mathbf{0}\in \mathcal{V}$, and suppose that $\phi$ is compatible with $\Gamma$. Let $\xi^*$ and $s_{\mathcal{C}}^*,\; \forall \mathcal{C}\in\Gamma -\{ \mathcal{V}\}$, be optimal for the minimization inside $\psi(X^\top\alpha)$ and define
\[\Delta_{\mathcal{V}} = X^\top \alpha - \sum_{\mathcal{C} \in \Gamma - \{{\mathcal{V}}\}} \lambda_{\mathcal{C}} s_{\mathcal{C}}^* - \xi^*.\]
If
\begin{equation}\label{eq:sparse}
    \lambda_{\mathcal{V}}^{-1}\Pi_{\mathcal{V}}\Delta_{\mathcal{V}} \in \textnormal{rel int}\left[{\mathcal{V}}\right],
\end{equation}
then $s_{{\mathcal{V}}}^* = \Pi_{\mathcal{V}}\Delta_{{\mathcal{V}}}/\lambda_{{\mathcal{V}}}$ is optimal for $\psi$, and all $\beta \in \beta(\alpha)$ computed using $\xi^*$ and the $s^*_{\mathcal{C}}$ satisfy $\Pi_{{\mathcal{V}}} \beta = \mathbf{0}$.
\end{thm}
The role of compatibility in this theorem is to ensure that any of the sparsity patterns $\Gamma$ can generate commute with the subdifferential operator $\partial \phi^*$; this guarantees that the sparsity pattern of primal coefficients matches the sparsity pattern of the argument to $\partial \phi^*$.

Theorem \ref{thm:dual_zeros} generalizes the basis for feature screening rules such as the SAFE rules \cite{safe_rules} which try to predict the sparsity pattern of primal solutions. For example, suppose that we know that some minimizer $\beta^*$ satisfies $\Pi_\mathcal{C}\beta^* = \mathbf{0}$ for all $\mathcal{C} \in \Omega \subset \Gamma$. This defines a subspace constraint, so we may restrict the search space to the subspace that is orthogonal to each of the $\Pi_\mathcal{C}$ in $\Omega$. If the columns of $U \in \mathbb{R}^{d \times p}$ define an orthonormal basis for this subspace, we may work with the reduced objective
\begin{equation}\label{eq:reduced_primal}
    \ell\left(XUz\right) +\sum_{\mathcal{C} \in \Gamma - \Omega} \lambda_{\mathcal{C}} \; \underset{s \in \mathcal{C}}{\text{max }} s^\top Uz + \phi\left(Uz\right)+ \omega\left(Uz\right).
\end{equation}
If $p \ll d$ this primal problem will be considerably faster to solve than the original. This reduced problem also yields a simpler $\psi$ function in the dual problem of Theorem \ref{thm:primal_dual}.

Theorem \ref{thm:dual_zeros} also helps with subgradient calculations at the current iterate $\alpha$. Assuming $\Omega\subset \Gamma$ is now defined for some $\beta \in \beta(\alpha)$ and $U$ is an orthonormal basis for the subspace orthogonal to all $\Pi_{\mathcal{C}}$ for $\mathcal{C} \in \Omega$ (in which $\beta$ lies), we have
\[X\beta = \left(XU\right)\left(U^\top \beta\right).\]
The quantity $u = U^\top \beta$ can be computed via the subgradient at $\alpha$ of the dual of equation (\ref{eq:reduced_primal}). Surprisingly, though the \emph{solution} of this reduced problem differs from the solution of the full problem, the dual subgradients of the two problems at this particular value of $\alpha$ are identical. Critically, the $\psi$ function pertaining to the reduced problem's dual uses as input $\left(XU\right)^\top \alpha$, obtained from the reduced feature matrix, instead of the potentially much larger $X^\top \alpha$. In the case of simple $\ell_1$ regularization, $U$ simply picks out the nonzero coefficients of $\beta$. For this reason, we will refer to $\left(XU\right)^\top \alpha$ as the \emph{effective support} of $X^\top \alpha$.

\section{Optimizing with Feature Interactions}\label{sec:interactions}
This section discusses how the effective support of $\mathscr{P}(A)^\top\alpha$ can be computed efficiently in the presence of sufficient sparsity. Our algorithm dynamically screens interactions to determine this support. We will use the case of $\ell_1$ regularization as a simple illustrative example throughout this section. We will index the columns of $\mathscr{P}\left(A\right)$ by subsets of $\{1,\dots,d\}$, so
$A_{[\![u]\!]} = \prod_{k \in u}A_k$
for $u \subset \{1,\dots,d\}$ denotes the elementwise product of all columns in $u$.

Our algorithm is based on the following theorem which generalizes the \emph{downward closure} property used by the Apriori algorithm \cite{mmds} for market basket analysis. We defer its proof to Appendix \ref{sec:app_thms}.
\begin{thm}\label{thm:any_vector}
    Given $\tau: 2^{\{1,\dots,d\}} \rightarrow \mathbb{R}$, $f: \mathbb{R}^n \rightarrow \mathbb{R}$, and $s \subset \{1,\dots,d\}$, if
    \begin{equation}\label{eq:zero_cond}
        \underset{\mathbf{0} \preceq x \preceq  A_{[\![s]\!]}}{\textnormal{max }} f(x) < \underset{t \succ s}{\textnormal{min }}\tau(t)
    \end{equation}
    then $f\left(A_{[\![t]\!]}\right) < \tau(t)$ for all $t \supset s$ where $t \subset \{1,\dots,d\}$.
\end{thm}
This theorem is intended to be used in conjunction with the sparsity check defined in equation (\ref{eq:sparse}) to prune out interactions that will obviously be excluded from the model. Typically $\tau$ will be a monotonically increasing set function, e.g. $\tau(t) \geq \tau(s)$ if $t \supset s$, for which the minimization on the right-hand side is easy to compute. The basic idea is that if for some $s\subset \{1,\dots,d\}$ we encounter an interaction column $A_{[\![s]\!]}$ that satisfies equation (\ref{eq:zero_cond}) then we can certify that all higher order interactions, involving the features in $s$ as well as others, can be ignored. 

For example, suppose that $\Gamma$ in equation (\ref{eq:primal}) defines a weighted $\ell_1$ regularizer
\[\sum_{\mathcal{C} \in \Gamma} \lambda_{\mathcal{C}} \; \underset{s \in \mathcal{C}}{\text{max }} s^\top\beta = \lambda\sum_{z \in 2^{\{1,\dots,d\}} - \{\emptyset\}} \rho(|z|)|\beta_z|
\]
where $\rho$ is monotonically increasing so interactions are increasingly penalized based on their order.
Assuming $\phi$ is compatible, Theorem \ref{thm:dual_zeros} shows that for any $\alpha$ and $u \subset \{1,\dots,d\}$ where $|A_{[\![u]\!]}^\top \alpha| < \lambda \rho(|u|)$, the members of $\beta \in \beta(\alpha)$ satisfy $\beta_u = 0$. Theorem \ref{thm:any_vector} extends this guarantee to \emph{all higher order interactions} $v \supset u$ provided 
\begin{equation}\label{eq:basic_check}
    \text{max}\left(A_{[\![u]\!]}^\top \alpha_+,\;A_{[\![u]\!]}^\top \alpha_- \right)< \lambda \rho(|u|+1).
\end{equation}
Here $\alpha = \alpha_+ - \alpha_-$ denotes the positive and negative components of $\alpha$, respectively.
\subsection{Algorithm Template}
Theorem \ref{thm:any_vector} translates into an algorithm to compute the support of $\mathscr{P}(A)^\top \alpha$ by iterating over the set of possible feature interactions in a bottom-up fashion. To make this precise, observe that the power set of $\{1,\dots,d\}$ forms a partially ordered set (poset) according to the subset relation $s \subset t \Leftrightarrow s \prec t$. The lowest nontrivial members of this poset are the singleton sets -- pertaining to the original columns of $A$, followed by quadratic interactions and so on. Screening rules like equation (\ref{eq:basic_check}) allow us to short circuit the search for which features will be included in the model by examining this poset bottom-up. Once the condition in Theorem \ref{thm:any_vector} is reached at a set $s$, all higher order $t \succ s$ can be ignored.

A wealth of algorithms \cite{market_basket} exist for efficiently searching posets in the market basket analysis literature. Algorithm \ref{alg:l1_main} incorporates the check from equation (\ref{eq:basic_check}) into the ECLAT algorithm \cite{eclat}, an alternative to Apriori, to screen interactions using $\ell_1$ regularization. We chose ECLAT because it is simpler to implement and faster than Apriori for market basket analysis. The difference between the two is that ECLAT performs a depth-first search of the interaction poset whereas Apriori is breadth-first. Algorithm \ref{alg:l1_main} will typically be initialized with a list of singleton sets $\{1\},\dots,\{d\}$ as inputs to scan all possible interactions. It ``emits" all interactions that can be formed as unions of its input sets and whose coefficient is not zeroed out by the Lasso. We emphasize that there are a variety of quantities the algorithm may actually output, including $A_{[\![u]\!]}$ to construct a cached matrix of relevant interactions or $A_{[\![u]\!]}^\top x$ to compute a matrix-vector product.
\begin{algorithm}[H]
\SetKwInOut{Input}{Input}
\SetKwInOut{Output}{Emits}
\Input{ List of sets of features to scan using dual variable $\alpha$}
\Output{ All interactions included by the Lasso that can be formed as unions of the input sets}
\SetKwIF{If}{ElseIf}{Else}{if}{}{else if}{else}{end if}%
\SetKwFunction{Screen}{Screen}
\SetKwProg{Fn}{Function}{:}{}
\SetKwInOut{Emit}{Emit}
\Fn{\Screen{$F, \alpha$}}{
    \For{$k\gets0$ \KwTo $|F|-1$}{
        $C\gets[\,]$\;
        \For{$j\gets k+1$ \KwTo $|F|$}{
            $u \gets F[k] \cup F[j]$\\
            \If{$|A_{[\![u]\!]}^\top \alpha| > \tau(u)$}{\KwSty{Emit} $u$\;}\label{line:emit}
            \If{$\textnormal{max}(A_{[\![u]\!]}^\top \alpha_+, A_{[\![u]\!]}^\top \alpha_-) > \underset{v \succ u}{\textnormal{min }}\tau(v)$}{Append $u$ to $C$}
        }
        \If{$C \neq [\,]$}{\Screen($C, \alpha$)}
    }
}
\caption{\label{alg:l1_main} Exact Feature Screening for $\ell_1$-norm}
\end{algorithm}
\subsubsection{Practical Complexity Considerations}\label{sec:practical}
An important aspect of equation (\ref{eq:basic_check}), and Theorem \ref{thm:any_vector} in general, is that it requires checking the positive and negative components of $\alpha$ separately. This is necessary because the model may include a higher order interaction $v \supset u$ without including its lower order ancestor. For example, suppose the positive and negative components of $\alpha$ \emph{individually} have a large inner product with $A_{[\![u]\!]}$ but annihilate each other in $|A_{[\![u]\!]}^\top \alpha|$. If $v$ zeros out the negative components of $\alpha$ so that $|A_{[\![v]\!]}^\top \alpha| \approx A_{[\![u]\!]}^\top \alpha_+$, then $v$ may pass the Lasso threshold. This annihilation can require checking many interactions in particularly recalcitrant learning problems where all lower order interactions are uncorrelated with the response. This is unlikely to be a problem in practical statistical settings where the need to scan many features to find one that passes the Lasso threshold indicates over-fitting; it is more likely that any included features passed this threshold by chance.

Including a $K^{\text{th}}$ order interaction into the model requires checking $2^K-1$ interactions because we must inspect all subsets of the $K$ features. In practice this is tractable up to $K\approx 20$ depending upon the dataset. For example, we accidentally included a $23^{\text{rd}}$ order interaction when working with the FDA drug interactions dataset below because a column was replicated $23$ times. This single interaction did not affect computation times appreciably.

The presence of many highly correlated columns in $A$, however, can cause substantial computational problems. If there are $N$ highly correlated features that all produce $K^{\text{th}}$ order interactions then $\sum_{k=1}^K{N \choose k}$ interactions must be scanned. When $N=100$ and $K=5$ the last term in this sum is approximately $75$ million. To this end we recommend combining highly correlated columns in $A$ by picking an arbitrary representative or averaging. It can also be useful to combine highly correlated interactions dynamically as they are explored by Algorithm \ref{alg:l1_main}, although care must be taken to combine features in a consistent manner. It is easy to include this dynamic clustering into the ECLAT algorithm. For example, in our experiments we modified line 6 to ignore an interaction if its support is too similar to its parents'. This method can be interpreted as approximately clustering features and deferring to a lower order representative. We leave it to future work to investigate different clustering schemes and representatives, such as averaging.
\subsection{Non-Negative Dual Coefficients}\label{sec:nn_coeff}
Equation (\ref{eq:basic_check}) simplifies considerably when the dual variable in problem {\ref{eq:dual}} is non-negative; the check reduces to
\begin{equation}\label{eq:nn_check}
    A_{[\![u]\!]}^\top \alpha < \lambda \rho(|u|+1).
\end{equation}
If $u$ does not satisfy this condition then $A_{[\![u]\!]}$ is \emph{necessarily} included in the model since $\rho(|u|) < \rho(|u|+1)$. This implies a \emph{strong hierarchy property} \cite{bien2013lasso} whereby any interaction $v \subset \{1,\dots,d\}$ is included in the model only if all lower order interactions $u\prec v$ are also included. Unlike the general Lasso case, this nonnegative dual regime places the number of interactions that are included in the model in tight correspondence with the number of interactions that are explored by Algorithm \ref{alg:l1_main}.

A natural question, answered by the theorem below, is which objectives ensure that the dual variable $\alpha$ is non-negative.
\begin{thm}\label{thm:nonneg_dual}
Under the assumptions of Theorem \ref{thm:primal_dual}, all optimal values of the dual variable $\alpha$ are non-negative if the loss function $\ell$ may be written as
\begin{equation}\label{eq:nonnegativity}
    \ell(p) = \underset{\xi \in \mathbb{R}_+^n}{\textnormal{min }}\ell(p - \xi) 
\end{equation}
for all $p$ in the domain of $\ell$. The converse is also true if $\ell$ is finite on at least two distinct points in $\mathbb{R}^n_{++}.$
\end{thm}
Unfortunately, this form of loss is limited; regression and classification losses do not have the requisite form. Losses that satisfy equation (\ref{eq:nonnegativity}) pertain to single class problems which try to ``cover'' data. Notable examples include one-class SVM's \cite{scholkopf2000support} and market basket analysis, whose primal objective may be written as
\begin{equation}\label{eq:market_basket}
        \underset{\mathbf{0} \preceq \beta \preceq \mathbf{1}}{\text{min }} -\mathbf{1}^\top \mathscr{P}(A)\beta + \lambda \mathbf{1}^\top \beta.
\end{equation}

\section{Experiments}\label{sec:experiments}
This section discusses three experiments we conducted to demonstrate the Fine Control Kernel framework's ability to train interaction models. Our goal is to demonstrate the capabilities of our algorithmic paradigm, not just a single algorithm, so we chose three learning problems that require disparate learning objectives and therefore different learning algorithms. Each experiment discusses the primal objective it solves. In all cases we solved the dual problem obtained from applying Theorem \ref{thm:primal_dual}; all derivations and further algorithmic details are relegated to Appendix \ref{sec:app_algos}. All relevant parameters were tuned over a grid of reasonable values and selected using a validation set, and accuracy on supervised problems is reported on a testing set. Working with the dual problem and predicting sparsity via Theorems \ref{thm:dual_zeros} and \ref{thm:any_vector} was critical for tractability. 

\subsection{FDA Drug Interactions}
Our first experiment focuses on the non-negative dual coefficient setting of section \ref{sec:nn_coeff}; we chose this first experiment because it provides tight control over the number of features that are scanned and the number of features that enter into the model. The objective we use is a ``democratic" alternative to market basket analysis that tries to balance the interactions included in the model by ensuring that they appear equitably among the training samples. This is achieved by squaring the loss of equation (\ref{eq:market_basket}) to obtain
    \begin{equation}\label{eq:l2_mb}
        \underset{\mathbf{0} \preceq \beta \preceq{1}}{\text{min }}\frac{1}{2}\left\|\tau \mathbf{1} - \mathscr{P}(A)\beta\right\|_+^2 + \lambda \mathbf{1}^\top \beta + \frac{\eta}{2}\|\beta\|_2^2.
    \end{equation}
Here $\|x\|_+^2 = \sum_{i: x_i > 0} x_i^2$ is invariant to \emph{overestimates}, so it satisfies the conditions of Theorem \ref{thm:nonneg_dual}. Parameter $\tau$, in conjunction with the upper bound on $\beta$, controls the order of interactions by setting the target number of nonzero features that should cover each training point. The $\ell_2$ norm on $\beta$ is there for algorithmic convenience; in practice it has little effect on the solution when sufficiently small.

We apply this objective to the problem of finding adverse interactions among drugs in the FDA's Adverse Event Reporting System (FAERS) \cite{faers}. Each entry in this database pertains to an adverse health event that is believed, by a physician, to have been caused by an interaction among a subset of the drugs listed. This data can be interpreted as a ``single class" problem in which we wish to find motifs -- combinations of drugs -- that explain as many of the events as possible. We use equation (\ref{eq:l2_mb}) to this end by defining its input feature matrix $A$ to be a binary indicator matrix in which each row pertains to an entry in the FAERS database, each column a distinct drug, and each entry an indicator for whether that drug was taken during the adverse event. The resulting matrix has nearly $9$ million rows (events) and $4,000$ columns (drugs) after normalizing drug names by RxNorm \cite{rxnorm}.

Our use of equation (\ref{eq:l2_mb}) is inspired by the use of market basket-type algorithms in multistage approaches for finding drug interactions \cite{dumouchel1999bayesian, fda_tech}. There a large list of candidate interactions is generated automatically, and this list is then pruned to a manageable size by a sequence of hand-tuned filters and ultimately examined manually. The precision of the candidate generation algorithm is therefore of critical importance. Figure \ref{fig:fda} compares the precision of equation (\ref{eq:l2_mb}) with $\tau = 10$ to market basket analysis as the number of proposals varies. We use as ground truth interactions deemed notable on the FAERS website. The proposals from equation (\ref{eq:l2_mb}) are substantially more precise than those of market basket analysis.

\begin{figure}[h!]
  \centering
  \includegraphics[scale=.3]{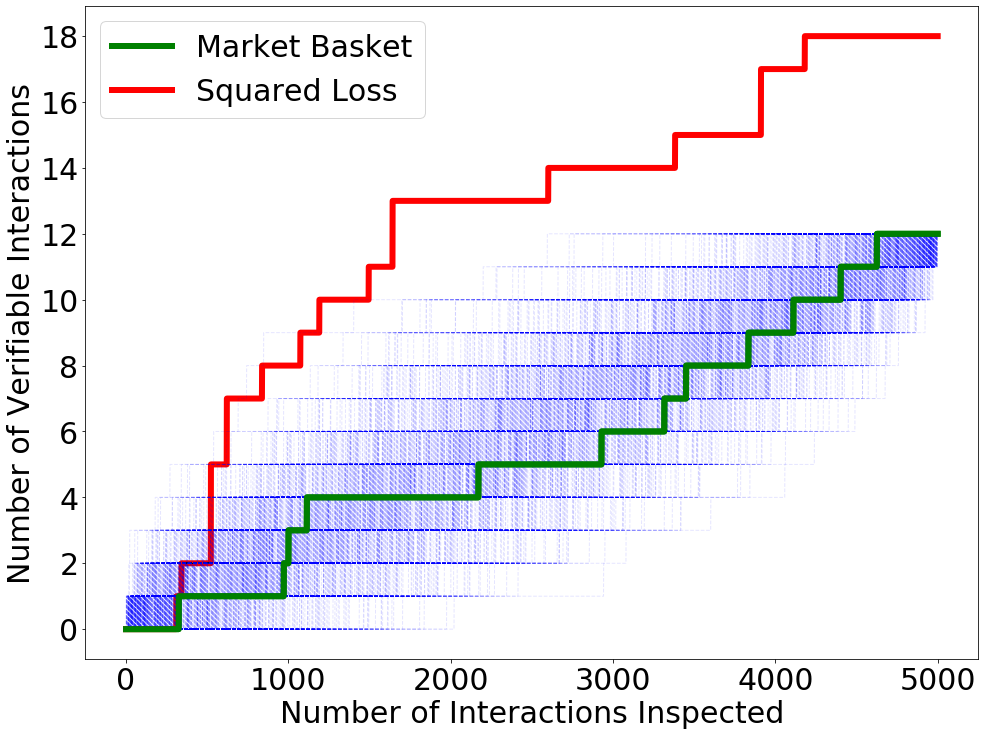}
  \caption{\label{fig:fda}Precision of drug interaction proposals by market basket analysis and ``squared loss" algorithm in equation (\ref{eq:l2_mb}). The blue lines are $1,000$ random permutations of market basket's output.}
\end{figure}

The regularization parameter $\lambda$ in equation (\ref{eq:l2_mb}) controls the number of proposed interactions, so Figure \ref{fig:fda} was generated by successively decreasing $\lambda$ and using the solution for $\lambda_{t-1}$, denoted by $\alpha(\lambda_{t-1})$, to warm-start the solver for $\lambda_t < \lambda_{t-1}$. Figure \ref{fig:homotopy} shows that this homotopy scheme allowed us to use $\alpha(\lambda_{t-1})$ to predict which interactions would be included in the primal solution at $\lambda_t$ with high fidelity. It underscores the importance of these warm starts since using $\alpha(\lambda_0)$, instead of $\alpha(\lambda_{t-1})$, to predict the sparsity pattern of the primal solution at $\lambda_t$ causes the optimizer to work with orders of magnitude more features. In this case, approximately tracing out the regularization path allows us to avoid the perils of working with a problem whose primal solution space is $2^{4,000}$-dimensional.
\begin{figure}[h!]
  \centering
  \includegraphics[scale=.3]{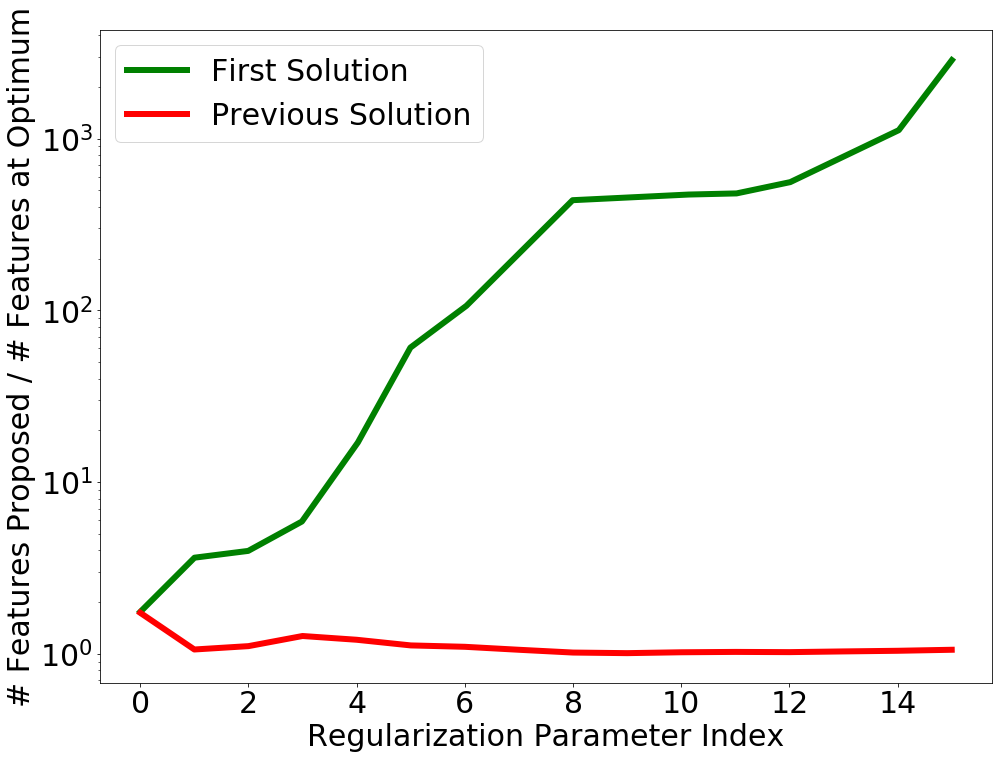}
  \caption{\label{fig:homotopy}Ratio of the number of interaction features \emph{predicted} in the primal solution for $\lambda_t$ to the true number of features in the primal solution for $\lambda_{t}$. Predictions are made either by using the \emph{previous solution} at $\lambda_{t-1}$ or the \emph{first solution} at $\lambda_0$. Predictions are always a superset of the true sparsity pattern.}
\end{figure}

\subsection{HIV Drug Screening}
Our next experiment involves a standard objective, Elastic-Net regularized logistic regression, using $\mathscr{P}(A)$ as the feature matrix. Following the discussion in section \ref{sec:practical}, this objective is markedly different from that of our previous experiment because the number of features included in the model no longer tightly controls the number of features that must be screened. While our algorithm successfully handled using a ``flat" $\ell_1$ regularization parameter $\lambda$ for all interactions, we achieved better predictive performance by penalizing the $\ell_1$-norm for $k^{\text{th}}$-order interactions according to $1.5^{k-1}\lambda$. This increasingly harsher penalty approximately controls for multiple hypothesis testing issues that arise when working with increasingly higher order features. We removed highly correlated features by greedily iterating through the columns of $A$ and ignoring all columns that are at least $99.9\%$ similar to a column that has already been seen. Finally, we found that pruning overly similar interactions by preferring lower order ones sped up training times without sacrificing accuracy. Specifically, we ignore an interaction involving features $u \cup \{i,j\}$ if its parents, $u \cup \{i\}$ and $u \cup \{j\}$, have a support that is more than $50\%$ similar to the child's. Nonetheless, the tuned models incorporated up to $5^{\text{th}}$-order features.

We applied our learning algorithm to the MoleculeNet HIV dataset \cite{moleculenet} which consists of nearly $33,000$ training examples and more than $4,100$ validation and testing samples. This data pertains to a binary classification task whose goal is to predict whether a particular molecule inhibits the replication of HIV. We experimented with two feature representations, the Extended-Connectivity Fingerprints (ECFP) representation \cite{moleculenet} as well as presence/absence indicators for $5$-grams extracted from the SMILE string, a human readable molecular descriptor, for each molecule. Table \ref{tab:hiv} presents the AUC achieved by our interactions-based logistic regression, XGBoost \cite{xgboost}, and a linear logistic regression on either feature representation. Surprisingly, the $5$-gram interactions model achieves top performance; this model is inherently interpretable since it indicates which substrings in the SMILE string are useful for its predictions.
\begin{table}
\caption{\label{tab:hiv}AUC on MoleculeNet HIV task.}
\vspace{1em}
\centering
\begin{tabular}{lll}
\hline
\textbf{Model}        & \textbf{AUC} & \textbf{SE}\\ \hline
Linear, ECFP          & 0.74     &   0.01   \\
Linear, 5-grams       & 0.745    &  0.01   \\
XGBoost, ECFP         & 0.75       &  0.01  \\
Interactions, ECFP    & \textbf{0.79}      &   0.01  \\
Interactions, 5-grams & \textbf{0.79}  &  0.01      
\end{tabular}
\end{table}
\subsection{Quantum Chemistry}
Our last experiment demonstrates our learning framework on a nuclear-norm and group-Lasso regularized matrix learning problem. Specifically, we solve
\begin{equation}\label{eq:matrix_problem}
\begin{aligned}
    \underset{\substack{W \in \mathbb{R}^{(2^d-1) \times T} \\ b \in \mathbb{R}^T}}{\text{min }} \frac{1}{2}\left\|Y - \mathscr{P}(A)W - \mathbf{1}b^\top \right\|_F^2 + R(W), \\ 
    R(W) = \rho \|\mathscr{P}(A)W\|_* + \lambda \|DW\|_{2,1} + \frac{\eta}{2}\|W\|_F^2
    \end{aligned}
\end{equation}
where $\|M\|_*$ denotes the matrix nuclear norm, $D$ is a diagonal matrix, and $\|M\|_{2,1}$ sums over the $\ell_2$-norm of the rows of $M$. As we discuss in Appendix \ref{sec:app_algos}, the nuclear norm of $\mathscr{P}(A)W$ also controls the rank of $W$. We improved prediction accuracy by aggressively penalizing high-order interactions via $D_{ii} = 1.5^{(k-1)^{1.5}}$, where $D_{ii}$ pertains to a $k^{\text{th}}$ order interaction, and we pruned interactions as in the HIV experiments. The resulting models included up to $4^{\text{th}}$-order interactions.

We applied our learning algorithm to the QM9 task in MoleculeNet \cite{moleculenet}, a dataset comprised of 134,000 samples split into $80\%$ training, $10\%$ validation, and $10\%$ testing samples. The goal of this task is to predict $12$ molecular properties typically found through expensive computations involving Schr\"{o}dinger's equation. We decided to use an unconventional approach to this task to highlight the power of interactions. We model the response $y \in \mathbb{R}^{12}$ for a particular molecule $m$ as
\[y \approx B\left(\text{coulomb}(m)\right) + C\left(\text{ngram}(m)\right)\]
where $B$ is a \emph{base} model that uses the Coulomb Matrix representation \cite{moleculenet} of $m$. Model $C$ is a \emph{correction} model that uses presence/absence indicators for $4$-grams extracted from the SMILE string of $m$. Critically, the Coulomb Matrix is computed from the same inputs used by Schr\"{o}dinger's equation, but the complexity of $B$ is restricted to constant or linear functions. The onus of capturing nonlinearities then falls on the correction model, which may only do so by leveraging interesting patterns from a human readable description of $m$. 

Table \ref{tab:quantum} compares the performance of ridge regression, XGBoost, and the model from equation (\ref{eq:matrix_problem}) when used as corrective models; the first row depicts the performance of the base models alone. Since $n$-grams capture different information than the Coulomb matrix the linear corrective model is able to improve performance. The interaction models use $\rho=0.01$ which results in a $9$-dimensional matrix (out of $12$ possible dimensions pertaining to each of the tasks). The $n$-gram features are able to substantially improve performance, and interactions further improve performance. It will be interesting in future work to explore how far we can improve the accuracy of this hybrid model, potentially by employing a better base model.

\begin{table}[]
\caption{\label{tab:quantum} Taskwide average $R^2$ on MoleculeNet QM9 task.}
\vspace{1em}
\centering
\begin{tabular}{lllll}
\hline
\multicolumn{1}{l|}{\multirow{2}{*}{\textbf{Correction Model}}} & \multicolumn{2}{l|}{\textbf{Constant Base}}           & \multicolumn{2}{l}{\textbf{Linear Base}}         \\ \cline{2-5} 
\multicolumn{1}{l|}{}                                & \multicolumn{1}{l|}{$R^2$} & \multicolumn{1}{l|}{SE} & \multicolumn{1}{l|}{$R^2$} & \multicolumn{1}{l}{SE} \\ \hline
None & 0 & 0 & 0.792 & 0.002 \\
Linear &0.722 &0.003 & 0.882 &0.002 \\
XGBoost &0.711 &0.004 &0.895 &0.001\\
Interactions  &0.785 &0.003   & \textbf{0.925} &  0.001                      
\end{tabular}
\end{table}
\section{Conclusion}
This work introduces a novel machine learning paradigm, the Fine Control Kernel framework, aimed at solving convex learning objectives that use massive, rich feature representations. This framework heavily leverages sparsity and is based on a restatement of Fenchel Duality specifically tailored to the structure of sparse learning problems. It makes use of the observation that massive feature matrices become eminently tractable in the presence of sparse regularization.

We use this framework to show how interactions models can be trained, even though the feature matrix $\mathscr{P}(A)$ pertaining to all possible multiplicative interactions among the columns of an atomic feature matrix $A$ is hopelessly large to access na\"{i}vely. This tractability is possible by way of feature screening rules that allow large swaths of interactions to be ignored because of sparsity. These screening rules are based on a generalization of the property that makes market basket analysis computationally feasible. Experiments with large datasets on a laptop demonstrate the computational viability and predictive accuracy of our approach.
\newpage
\appendix
\section*{Appendix}
\renewcommand{\thesubsection}{\Alph{subsection}}
\subsection{Theorem Proofs}\label{sec:app_thms}
This section proves the theorems stated in the main paper; we restate the theorems here for ease of exposition. Our first Theorem provides a duality statement for problems of the form
\begin{equation}\label{eq:primal_}
    \underset{\beta \in \mathbb{R}^d}{\text{min }} \ell\left(X\beta\right) +\left[ \sum_{\mathcal{C} \in \Gamma} \lambda_{\mathcal{C}} \; \underset{s \in \mathcal{C}}{\text{max }} s^\top\beta + \phi\left(\beta\right)+ \omega\left(\beta\right)\right].
\end{equation}
We first the prove the following Lemma, which provides a connection between the subgradients of sums of functions and the infimal convolution of those functions. 

\begin{lem}\label{lem:subdifferential_}
Let $f_1,\dots,f_m$ be closed proper convex functions whereby the sets $\textnormal{rel int } \left(\textnormal{dom } f_i\right)$ have a nontrivial intersection. Suppose that $x,x^*$ satisfy
\begin{equation}\label{eq:optimal_basic_}
x^* \in \partial f_1(x) + \dots + \partial f_m(x)
\end{equation}
and let $x^*_i \in \partial f_i(x)$ for $i=1,\dots,m$ be chosen so that $x^* = \sum_{i=1}^m x_i^*$. Then the $x^*_i$ are minimizers of
\begin{equation}\label{eq:inf_conv_}
\begin{aligned}
&\underset{x_i', \; i = 1,\dots,m}{\textnormal{minimize }}& f_1^*(x_1') + \dots + f_m^*(x_m')\\
&\textnormal{ subject to } & x_1' + \dots + x_m' = x^*.
\end{aligned}
\end{equation}
\end{lem}

\begin{proof}
The condition in equation (\ref{eq:optimal_basic_}) pertains to the optimality condition for the problem
\begin{equation}
\underset{z}{\text{max }} z^\top x^* - \sum_{i=1}^mf_i(z),
\end{equation}
with $x$ an optimizer of this problem.
We may thus write
\begin{equation}
\begin{aligned}
\sum_{i=1}^m f_i^*\left(x_i^*\right) &=& \sum_{i=1}^m \left(x^\top x_i^* - f_i(x)\right) = x^\top x^* - \sum_{i=1}^mf_i(x) \\
&=& \underset{z}{\text{max }} z^\top x^* - \sum_{i=1}^mf_i(z) = \left(\sum_{i=1}^mf_i\right)^*\left(x^*\right) \\
&=& \underset{x_1' + \dots + x_m' = x^*}{\text{min }}\sum_{i=1}^m f^*_i(x_i').
\end{aligned}
\end{equation}
The first equality follows from Theorem 23.5.d of \cite{convex}, the second from the fact that the $x_i^*$ sum to $x^*$, the third from observing that $x$ is optimal for the maximization problem, the fourth from observing that this is the Fenchel conjugate of the sum of the $f_i$ evaluated at $x^*$, and the final equality from Theorem 16.4 of \cite{convex}. The Lemma follows by noting that $x_i^*$ satisfy the linear constraint of the infimal convolution.
\end{proof}

We may now prove our first theorem, which is restated below.

\begin{thm}\label{thm:primal_dual_}
Suppose that the objective in equation (\ref{eq:primal_}) is strictly feasible, bounded from below, and $\ell \circ X, \phi,\omega$ are closed proper convex functions. Define
\[\psi(\zeta)=\underset{\substack{s_{\mathcal{C}} \in \mathcal{C}, \forall \mathcal{C} \in \Gamma \\ \xi \in \mathbb{R}^d}}{\textnormal{min }}\left[\phi^*\left(\zeta - \sum_{\mathcal{C} \in \Gamma} \lambda_{\mathcal{C}} s_{\mathcal{C}} - \xi\right) + \omega^*(\xi) \right].\]
Then
\begin{equation}\label{eq:dual_}
\begin{aligned}
    &\underset{\alpha \in \mathbb{R}^n}{\textnormal{max }} -\ell^*(-\alpha) - \psi\left(X^\top \alpha\right)
\end{aligned}
\end{equation}
is a dual problem. Any dual feasible $\alpha$ determines a candidate set of primal coefficients $\beta(\alpha)$ as the subdifferential
\begin{equation}\label{eq:beta_form_}
    \beta(\alpha) = \partial \phi^*\left(X^\top\alpha - \sum_{\mathcal{C} \in \Gamma} \lambda_{\mathcal{C}}s_{\mathcal{C}} - \xi\right)
\end{equation}
where the $s_{\mathcal{C}}$ and $\xi$ are optimal for the minimization inside $\psi$.
At least one optimal primal solution exists, and $\beta,\alpha$ form an optimal primal-dual pair iff. $\beta \in \beta\left(\alpha\right)$ and $-\alpha \in \partial \ell(X\beta)$.
\end{thm}

\begin{proof}
Consider the convex function $f: \mathbb{R}^n \times \mathbb{R}^d\rightarrow \mathbb{R}$
\begin{equation}
f\left(\begin{bmatrix}p \\ \beta\end{bmatrix}\right) = \ell\left(p\right) + \sum_{\mathcal{C} \in \Gamma} \lambda_{\mathcal{C}} \; \underset{s \in \mathcal{C}}{\text{max }} s^\top \beta + \phi\left(\beta\right)+ \omega\left(\beta\right).
\end{equation}
and the subspace
\[L = \left\{p \in \mathbb{R}^n, \beta \in \mathbb{R}^d \mid p = X\beta\right\}.\]
Clearly minimizing $f$ over the subspace $L$ is equivalent to the problem in equation (\ref{eq:primal_}). The subspace condition may be expressed in the form $Ax = \mathbf{0}$ where
\[A = 
\begin{bmatrix}
 -I & X
\end{bmatrix}, \;\; 
x = \begin{bmatrix}
p \\ \beta
\end{bmatrix}.\]
We may now appeal to Fenchel's Duality Theorem, e.g. Theorem 31.1 in \cite{convex}, to conclude that
\begin{equation}\label{eq:basic_format_}
\begin{aligned}
\underset{x }{\text{min }} &f(x) - \left[-I\left(Ax = \mathbf{0}\right)\right] = \underset{\alpha \in \mathbb{R}^n}{\text{max }}-f^*(A^\top \alpha).
\end{aligned}
\end{equation}
Here we have simplified the maximization problem by noting that the Fenchel conjugate of $I\left(Ax = \mathbf{0}\right)$ restricts the solution to lie in the row space of $A$.


Function $f$ is separable according t   o the $p,\beta$ variables, so its Fenchel conjugate is too. In particular, we may write
\begin{equation}
    f^*\left(p^*, \beta^*\right) = \ell^*(p^*) + \left(\phi + \omega + \sum_{\mathcal{C} \in \Gamma} \lambda_{\mathcal{C}} \; \underset{s \in \mathcal{C}}{\text{max }} s^\top\right)^*(\beta^*).
\end{equation}
The second conjugate can be rewritten as the infimal convolution
\begin{equation*}
\begin{aligned}
&\left(\phi + \omega + \sum_{\mathcal{C} \in \Gamma} \lambda_{\mathcal{C}} \; \underset{s \in \mathcal{C}}{\text{max }} s^\top \right)^*(\beta^*) \\=
& \underset{\substack{\beta_{\mathcal{C}} \in \mathbb{R}^d, \; \forall\mathcal{C} \in \Gamma \\ \xi \in \mathbb{R}^d}}{\text{min }}\phi^*\left(\beta^* - \sum_{\mathcal{C} \in \Gamma}\beta_{\mathcal{C}} - \xi\right) + \omega^*(\xi)+ \sum_{\mathcal{C} \in \Gamma}\left(\lambda_{\mathcal{C}} \; \underset{s \in \mathcal{C}_{\mathcal{G}}}{\text{max }} s^\top\right)^*\left(\beta_{\mathcal{C}}\right)\\ =
& \underset{\substack{s_{\mathcal{C}} \in \mathcal{C}, \; \forall\mathcal{C} \in \Gamma \\ \xi \in \mathbb{R}^d}}{\text{min }}\phi^*\left(\beta^* - \sum_{\mathcal{C} \in \Gamma} \lambda_{\mathcal{C}}s_{\mathcal{C}} - \xi\right) + \omega^*\left(\xi\right) = \psi(\beta^*).
\end{aligned}
\end{equation*}
The first equality follows by noting that $x$ is a maximizer of the Fenchel conjugate
\begin{equation*}
    \underset{x \in \mathbb{R}^d}{\text{max }}x^\top \beta^* - \left(\phi + \omega + \sum_{\mathcal{C} \in \Gamma} \lambda_{\mathcal{C}} \; \underset{s \in \mathcal{C}}{\text{max }} s^\top \right)^*(\beta^*)
\end{equation*}
only if
\[\beta^* \in \partial \phi(x) + \partial \omega(x)  + \sum_{\mathcal{C} \in \Gamma} \partial \left(\lambda_{\mathcal{C}} \underset{s \in \mathcal{C}}{\text{max }}s^\top\right)(x).\]
Applying Lemma \ref{lem:subdifferential_} to this condition and eliminating the optimization variable pertaining to $\phi^*$ provides the requisite form. The second equality follows from the duality between indicator and support functions for convex sets, i.e.
\[\left(\underset{s \in \mathcal{C}}{\text{max }} \lambda_{\mathcal{C}} s^\top\right)^*(x) = \left( \underset{s}{\text{max }} \lambda_{\mathcal{C}}s^\top - I\left(s \in \mathcal{C}\right) \right)^*(x) = I\left(\lambda^{-1}_{\mathcal{C}}x \in \mathcal{C}\right).\]
Finally, plugging these conjugates and matrix $A$ into equation (\ref{eq:basic_format_}) yields the desired form of equation (\ref{eq:dual_}).

Equation (\ref{eq:basic_format_}) establishes primal-dual problems with zero duality gap at their respective optimizers. Theorem 31.3 establishes necessary and sufficient conditions for $\beta,\alpha$ to be an optimal primal-dual pair by noting that $x^* = A^\top \alpha$ in equation (\ref{eq:basic_format_}) which leads to the conditions
\[\begin{bmatrix}-\alpha \\ X^\top \alpha\end{bmatrix} \in \begin{bmatrix} \partial\ell(X\beta)\\ \partial\left(\phi + \omega + \sum_{\mathcal{C} \in \Gamma} \lambda_{\mathcal{C}} \; \underset{s \in \mathcal{C}}{\text{max }} s^\top \right)(\beta) \end{bmatrix}\]
as well as the vacuous condition $X\beta = X\beta$.
By 23.5a and 23.5.a* in \cite{convex} the condition on $X^\top \alpha$ is equivalent to
\[X^\top \alpha \in \partial\left(\phi + \omega + \sum_{\mathcal{C} \in \Gamma} \lambda_{\mathcal{C}} \; \underset{s \in \mathcal{C}}{\text{max }} s^\top \right)(\beta) \Leftrightarrow \beta \in \partial \psi\left(X^\top\alpha\right).\]
It remains to express this latter subdifferential in a more manageable format to finish the theorem. To this end, by Theorem 23.5 of \cite{convex} observe that the subdifferential of $\psi$ is characterized by the set of all pairs $u, \beta$ where $u$ is optimal for
\[\underset{u}{\text{max }} u^\top\beta - \psi(u) = \underset{u}{\text{max }} u^\top\beta + \underset{\substack{\beta_{\mathcal{C}} \\ \xi}}{\text{max }} -\left[\phi^*\left(u - \sum_{\mathcal{C} \in \Gamma} \lambda_{\mathcal{C}}s_{\mathcal{C}} - \xi\right) + \omega^*\left(\xi\right)\right]\]
\[=\underset{\substack{\beta_{\mathcal{C}} \\ u, \xi}}{\text{max }} u^\top \beta - \phi^*\left(u - \sum_{\mathcal{C} \in \Gamma} \lambda_{\mathcal{C}}s_{\mathcal{C}} - \xi\right) - \omega^*\left(\xi\right).\]
The optimality conditions for $u$ in this last problem lead to
\[\beta = \partial \phi^*\left(u - \sum_{\mathcal{C} \in \Gamma} \lambda_{\mathcal{C}}s_{\mathcal{C}} - \xi\right)\]
where the $s_{\mathcal{C}}$ and $\xi$ are optimal in $\psi$ for $u$. Substituting $u = X^\top \alpha$ yields the requisite formula for $\partial \psi(X^\top \alpha)$.
\end{proof}

\begin{thm}\label{thm:dual_zeros_}
Fix $\alpha \in \mathbb{R}^n$, $\mathcal{V} \in \Gamma$ with $\mathbf{0}\in \mathcal{V}$, and suppose that $\phi$ is compatible with $\Gamma$. Let $\xi^*$ and $s_{\mathcal{C}}^*,\; \forall \mathcal{C}\in\Gamma -\{ \mathcal{V}\}$, be optimal for the minimization inside $\psi(X^\top\alpha)$ and define
\[\Delta_{\mathcal{V}} = X^\top \alpha - \sum_{\mathcal{C} \in \Gamma - \{{\mathcal{V}}\}} \lambda_{\mathcal{C}} s_{\mathcal{C}}^* - \xi^*.\]
If
\begin{equation}\label{eq:sparse_}
    \lambda_{\mathcal{V}}^{-1}\Pi_{\mathcal{V}}\Delta_{\mathcal{V}} \in \textnormal{rel int}\left[{\mathcal{V}}\right],
\end{equation}
then $s_{{\mathcal{V}}}^* = \Pi_{\mathcal{V}}\Delta_{{\mathcal{V}}}/\lambda_{{\mathcal{V}}}$ is optimal for $\psi$ and all $\beta \in \beta(\alpha)$ computed using $\xi^*$ and the $s^*_{\mathcal{C}}$ satisfy $\Pi_{{\mathcal{V}}} \beta = \mathbf{0}$.
\end{thm}

\begin{proof}
We directly verify that $s^*_\mathcal{V}$ as defined is optimal. Since $\mathbf{0} \in \mathcal{V}$, the affine space containing $\mathcal{V}$ is a subspace so $\Pi_{\mathcal{V}}$ is linear projector. Thus, the quantity
\[\Pi_{\mathcal{V}}\left(X^\top \alpha - \sum_{\mathcal{C} \in \Gamma - \{{\mathcal{V}}\}} \lambda_{\mathcal{C}} s_{\mathcal{C}}^* - \xi^* - \lambda_{\mathcal{V}}s^*_{\mathcal{V}}\right) = \Pi_{\mathcal{V}}\Delta_{\mathcal{V}} - \Pi_{\mathcal{V}}\Delta_{\mathcal{V}} = \mathbf{0}.\]
Theorem 23.5 of \cite{convex} shows that $v \in \partial \phi(u)$ if and only if $u \in \partial \phi^*(v)$. Thus, since $\phi$ is compatible with $\Gamma$, we have that all
\[\beta \in \partial \phi^*\left(\Delta_{\mathcal{V}} - \lambda_{\mathcal{V}}s^*_{\mathcal{V}}\right) = \beta(\alpha)\]
satisfy $\Pi_{\mathcal{V}} \beta = \mathbf{0}$. This directly proves the second assertion of the Theorem. It also proves the optimality of $s^*_\mathcal{V}$ by Theorem 27.4 in \cite{convex} since all nonzero subgradients of $\phi^*\left(\Delta_{\mathcal{V}} - \lambda_{\mathcal{V}}s^*_{\mathcal{V}}\right)$ are orthogonal to the subspace containing $\mathcal{V}$. 
\end{proof}


\begin{thm}\label{thm:any_vector_}
    Given $\tau: 2^{\{1,\dots,d\}} \rightarrow \mathbb{R}$, $f: \mathbb{R}^n \rightarrow \mathbb{R}$, and $s \subset \{1,\dots,d\}$, if
    \begin{equation}\label{eq:zero_cond_}
        \underset{\mathbf{0} \preceq x \preceq  A_{[\![s]\!]}}{\textnormal{max }} f(x) < \underset{t \succ s}{\textnormal{min }}\tau(t)
    \end{equation}
    then $f\left(A_{[\![t]\!]}\right) < \tau(t)$ for all $t \supset s$ where $t \subset \{1,\dots,d\}$.
\end{thm}

\begin{proof}
Fix $t \supset s$. Observe that $\mathbf{0} \preceq A_{[\![t]\!]} \preceq A_{[\![s]\!]}$ and therefore
\[f\left(A_{[\![t]\!]}\right) \leq \underset{\mathbf{0} \preceq x \preceq  A_{[\![s]\!]}}{\textnormal{max }} f(x) < \underset{t' \succ s}{\textnormal{min }}\tau(t') \leq \tau(t).\]
\end{proof}


\begin{thm}
Under the assumptions of Theorem \ref{thm:primal_dual_}, all optimal values of the dual variable $\alpha$ are non-negative if the loss function $\ell$ may be written as
\begin{equation}\label{eq:nonnegativity_}
    \ell(p) = \underset{\xi \in \mathbb{R}_+^n}{\textnormal{min }}\ell(p - \xi). 
\end{equation}
for all $p$ in the domain of $\ell$. The converse is also true if $\ell$ is finite on at least two distinct points in $\mathbb{R}^n_{++}.$
\end{thm}
\begin{proof}
In what follows we assume that $f$ has a nontrivial domain since otherwise the proof is trivial. 
Suppose that equation (\ref{eq:nonnegativity_}) holds. Then
\[\ell^*(u) = \underset{x}{\text{max }} u^\top x - \underset{\xi \in \mathbb{R}^n_+}{\text{min }} \ell(x-\xi) = \begin{cases} +\infty & \text{if } u \notin \mathbb{R}_-^n \\ \ell^*(u) & \text{otherwise}.\end{cases}\]
The first equality follows from Corollary 12.2.2 of \cite{convex} and equation (\ref{eq:nonnegativity_}). The second equality follows from the assumption that $\ell$ has a non-trivial domain so that there is some $p$ where $f(p) < \infty$. If there is some coordinate $u_i > 0$ we may set $x_j = p_j$ and $\xi_j^* = 0$ for all $j \neq i$ and $x_i = p_i + v$ and $\xi_i^* = v$. As $v \uparrow +\infty$ the inner product $u^\top x$ grows unboundedly whereas $\underset{\xi \in \mathbb{R}^n_+}{\text{min }} \ell(x-\xi) \leq \ell(x - \xi^*)$, so $l^*(u)$ grows unboundedly. Next, suppose that $u \preceq \mathbf{0}$. For any point $p = x - \xi$ where $\xi \neq \mathbf{0}$ we have
\[u^\top p - \ell(p) = u^\top x - \ell(p) - u^\top\xi \geq u^\top x - \ell(p)\]
since $- u^\top\xi \geq 0$. Since we can do no worse if we set $\xi = \mathbf{0}$, it follows that $\underset{\xi \in \mathbb{R}^n_+}{\text{min }} \ell(x-\xi) = \ell(x)$ when $u \preceq \mathbf{0}$.

For the converse, we first prove that $\text{relint dom  }\ell$ nontrivially intersects $\mathbb{R}_{++}^n$. Let $\mathcal{H}$ denote the affine hull of $\text{dom }\ell$ and let $\mathcal{B}$ be a bounded open ball in $\mathcal{H}$. Such a ball exists because our assumption that $\ell$ is finite on at least two points in $\mathbb{R}^n_{++}$ immediately implies that its domain has a non-trivial relative interior. Let $p \in \mathbb{R}_{++}^n$ be a point where $\ell(p) < \infty$; such a point exists by assumption. Let $\varepsilon > 0$ be its smallest entry. Define 
\[M = \underset{x \in \mathcal{B}}{\text{sup }} \|p - x\|_{\infty}\]
which satisfies $\|p-x\|_{\infty} < M < \infty$ for all $x\in\mathcal{B}$ because $\mathcal{B}$ is bounded and open. Consider the ball 
\[\mathcal{B}' = \sigma \mathcal{B} + (1-\sigma)p \text{ where } \sigma =  \frac{\text{min}\left(\varepsilon, M\right)}{M}.\]
Since $M < \infty$ this provides a one-to-one continuous mapping between points in $\mathcal{B}$ and $\mathcal{B}'$, so the latter ball is also open in $\mathcal{H}$. Now consider some $z \in \mathcal{B}$ and its corresponding point $z'=\sigma z + \left(1 - \sigma\right)p$ in $\mathcal{B}'$. We have that
\[\left\|z' - p\right\|_{\infty} = \left\|\sigma z + \left(1 - \sigma\right)p - p\right\|_{\infty} = \sigma \left\|p - z\right\|_{\infty} < \varepsilon.\]
Since any point satisfying $\left\|z' - p\right\|_{\infty} < \varepsilon$ must have strictly positive entries, $\mathcal{B}' \subset \mathbb{R}^n_{++}$. Furthermore by convexity $\mathcal{B}' \subset \text{dom }\ell$, so $\text{relint dom  }\ell \cap \mathbb{R}^n_{++} \neq \emptyset$.

Now suppose that the optimal dual variable $\alpha$ is always non-negative. We then have the following chain of equalities.
\begin{subequations}
\begin{align}
\ell(p)& \\
&=  \left(\ell^*\right)^*(p) \label{eq:4_1_}\\
&=  \left(\ell^* + I(\mathbb{R}^n_-)\right)^*(p) \label{eq:4_2_}\\
&=\underset{x }{\text{max }} p^\top x - \ell^*(x) - I\left(x \preceq \mathbf{0}\right)\label{eq:4_3_}\\
&=\underset{x,y }{\text{max }} p^\top x - \ell^*(x) - I\left(y \preceq \mathbf{0}\right) - I\left(x=y\right)\label{eq:4_4_}\\
&= \underset{\xi}{\text{min }} \left(l^* - p^\top\right)^*(-\xi) + I(\xi \succeq \mathbf{0})\label{eq:4_5_}\\
&=\underset{\xi \in \mathbb{R}_+^n}{\text{min }}\ell(p - \xi)\label{eq:4_6_}
\end{align}
\end{subequations}
Equation (\ref{eq:4_1_}) follows from Theorem 12.2 of \cite{convex} since $f$ is convex, proper, and closed. Equation (\ref{eq:4_2_}) follows from our assumption whereby adding a non-negativity constraint on $\alpha$ does not change the dual. Note that we pass $-\alpha$ into $l^*$, so $\left(l^* + I\{\mathbb{R}_-^n\}\right)(-\alpha)$ ensures non-negativity of $\alpha$. Equation (\ref{eq:4_3_}) follows from Corollary 12.2.2 of \cite{convex}. Equation (\ref{eq:4_4_}) is obtained by introducing an auxiliary variable to construct an equivalent problem. Equation (\ref{eq:4_5_}) follows from applying Theorem 31.1 in \cite{convex} and simplifying. Here we use the definitions
\[g(x,y) = p^\top x - l^*(x) - I(y \preceq \mathbf{0}),\; f(x,y) = I(x=y).\]
This step makes use of the fact that $g,f$ as defined are closed so we may employ condition $(b)$ in the Theorem with our earlier observation that $\text{relint dom  }\ell \cap \mathbb{R}^n_{++} \neq \emptyset$. Finally, equation (\ref{eq:4_6_}) uses the closedness of $l^*$ to appeal to Theorem 12.2. of \cite{convex}.
\end{proof}
\subsection{Fenchel Conjugate Reference}\label{sec:app_conj}
This section provides a table of sample Fenchel conjugates that are useful for the Appendix. Many of these are standard results or otherwise appear in the literature. We provide proofs to demonstrate the simplicity of this operation.
\begin{table}
\begin{tabular}{ll}
\hline
\textbf{Function} $f(x)$& \textbf{Fenchel Conjugate} $f^*(\alpha)$ \\[1mm] \hline \hline 
\textbf{Specific Functions}& \\[1mm]\hline
$\frac{\rho}{2}\left\|y - x\right\|_2^2$ & $\frac{1}{2\rho}\left\|\frac{1}{\rho}\alpha + y\right\|_2^2 - \frac{\rho}{2}\|y\|_2^2$ \\[1mm]
$0$ & $I\left\{\alpha = \mathbf{0}\right\}$ \\[1mm]
$I(x \in \mathcal{C})$ & $\underset{x \in \mathcal{C}}{\text{max }} x^\top \alpha$\\[1mm]
$-yx + \log \left(1 + e^x\right)$ & $\begin{cases}\left(\alpha + y\right) \log (\alpha + y) - \left(1 - \alpha - y\right) \log(1 -\alpha - y) & \text{if } 0\leq \alpha + y \leq 1 \\ \infty & \text{otherwise.}\end{cases}$\\[1mm]
\textbf{General Operations}&\\[1mm]\hline
$\underset{v \in \mathcal{C}}{\text{min }}g\left(x - v\right)$ &$g^*(\alpha) + \underset{v \in \mathcal{C}}{\text{max }}\alpha^\top v$ \\[1mm]
\end{tabular}
\caption{Table of Sample Fenchel Conjugates}
\label{tab:conjugates_}
\end{table}
\newpage
\begin{proof}
\begin{enumerate}
    \item Direct computation yields
    \[\underset{x \in \mathbb{R}^n}{\text{max }} x^\top \alpha - \frac{\rho}{2}\|y - x\|_2^2 = \underset{x \in \mathbb{R}^n}{\text{max }} x^\top \left(\alpha + \rho y\right) - \frac{\rho}{2}\|y \|_2^2-\frac{\rho}{2}\|x\|_2^2\]
    The optimal $x$ is given by $x = \frac{1}{\rho}\alpha + y$ yielding
    \[\frac{1}{2\rho}\left\|\frac{1}{\rho}\alpha + y\right\|_2^2 - \frac{\rho}{2}\|y\|_2^2.\]
    \item Trivial.
    \item Trivial.
    \item The conjugate requires solving
    \[\underset{x}{\text{max }} \alpha x + yx - \log \left(1 + e^x\right),\]
    whose optimality conditions require
    \[\alpha + y = \frac{e^x}{1 + e^x}.\]
    Defining $z = e^x$ and simplifying yields the condition
    \[z = \frac{\alpha + y}{1 -\alpha - y} \Rightarrow x = \log \frac{\alpha + y}{1 -\alpha - y}.\]
    Note that $x$ is only defined if $\alpha + y \in (0, 1)$.
    Plugging back in and simplifying provides the conjugate 
    \[\left(\alpha + y\right) \log (\alpha + y) - \left(1 - \alpha - y\right) \log(1 -\alpha - y),\]
    which is defined on $0 \leq \alpha + y \leq 1$ with the endpoints verified by l'Hopital's rule.
    \item Conjugation requires solving
    \[\underset{x}{\text{max }}x^\top\alpha - \underset{v \in \mathcal{C}}{\text{min }}g\left(x - v\right) = \underset{x,v}{\text{max }}x^\top\alpha - I(v \in \mathcal{C})-g\left(x - v\right).\]
    Now define $u=x-v$ to obtain the equivalent problem
    \[\underset{u,v}{\text{max }}(u+v)^\top\alpha - I(v \in \mathcal{C})-g\left(u\right) = g^*(\alpha) + \underset{v \in \mathcal{C}}{\text{max }}\alpha^\top v.\]

\end{enumerate}
\end{proof}

\subsection{Algorithm Derivations}\label{sec:app_algos}
This section provides a discussion of the algorithms derived to optimize the objectives in the Experiments section of the main paper. The algorithms we construct are quasi-Newton methods that stem from the observation that if $f$ is a differentiable proper convex function, then for some $h$ sufficiently large the minimizer
\begin{equation}\label{eq:quasi_newton_}
    x^* = \underset{x \in \mathcal{S}}{\text{arg min }} \braket{x-\bar{x},\nabla f(\bar{x})} + \frac{1}{2} \braket{x-\bar{x}, \left(H + hI\right)\left(x-\bar{x}\right)}
\end{equation}
yields $f(x^*) < f(\bar{x})$ provided that
\begin{enumerate}
\item $\nabla f(\bar{x}) \neq \mathbf{0}$,
\item $H$ is positive semidefinite, and
\item $\mathcal{S} \cap \mathcal{B}\left(\bar{x}, \varepsilon\right)$ contains a point $x$ with $f(x) < f(\bar{x})$ for all $\varepsilon > 0$.
\end{enumerate}
This methodology subsumes gradient descent algorithms (e.g. set $H = \mathbf{0}$), and a simple back-off strategy that switches to a gradient descent scheme for steps that fail to make sufficient progress ensures that we can do \emph{no worse} than a gradient descent scheme. Moreover, it yields considerable flexibility by allowing the derivation of effective quasi-Newton methods even when $f$ is not second order differentiable or when $\mathcal{S}$ is not convex simply by specifying reasonable choices for $H$. In all cases below we use a conjugate gradient solver \cite{nocedal} to solve linear equations involving $H + hI$.

\subsubsection{Squared Market Basket Analysis}
The objective pertaining to Market Basket Analysis may be written as
\begin{equation}
    \underset{\beta \in [0,1]^d}{\text{minimize }} -\mathbf{1}^\top X \beta + \lambda \mathbf{1}^\top \beta.
\end{equation}
Note that this satisfies the conditions for a nonnegative dual variable since the loss function is $-\mathbf{1}^\top$ so that adding a $\xi$ variable yields the loss
\[\ell(p)=\underset{\xi \in \mathbb{R}^n_+}{\text{min }} -\mathbf{1}^\top \left(p - \xi\right)=\underset{\xi \in \mathbb{R}^n_+}{\text{min }} -\mathbf{1}^\top p + \mathbf{1}^\top \xi\]
which is always optimized by setting $\xi = \mathbf{0}$.
For a given value of $\lambda$ the optimality conditions show that
\[\beta_i = \begin{cases} 0 \text{ if } \mathbf{1}^\top X_i > \lambda \\ 1 \text{ otherwise.}\end{cases}\]
Computing the regularization path of this objective yields an A-Priori type algorithm.

Squaring the loss yields a more democratic objective which encourages sample points to be covered by non-zero coefficients. The objective we will work with is
\begin{equation}
    \underset{\substack{\beta \in [0,1]^d \\ \xi \in \mathbb{R}_+^n}}{\text{minimize }} \frac{1}{2}\left\|\tau\mathbf{1} - X \beta + \xi\right\|_2^2 + \lambda \mathbf{1}^\top \beta + \frac{\gamma}{2}\left\|\beta\right\|_2^2 
\end{equation}
We have added an arbitrarily small amount of $\ell_2$ regularization to ensure solution uniqueness, as could happen with multiple copies of the same columns in $X$.
In order to apply Theorem \ref{thm:primal_dual_} we make the following identifications; when applicable we also provide the Fenchel conjugates derived from table \ref{tab:conjugates_}.
\begin{enumerate}
    \item $\ell(p) = \underset{\xi \in \mathbb{R}_+^n}{\text{min }} \frac{1}{2}\left\|\tau \mathbf{1} + \xi - p \right\|_2^2 \; \Rightarrow \; \ell^*(\alpha) = \frac{1}{2}\left\|\tau \mathbf{1} + \alpha\right\|_2^2 - \frac{\tau^2 n}{2} + I(\alpha \in \mathbb{R}^n_{+})$
    \item $\Gamma = \left\{\{\mathbf{1}\}, \mathbb{R}^d_-\right\}$
    \item $\phi(\beta) = \frac{\gamma}{2}\left\|\beta\right\|_2^2 \; \Rightarrow \; \phi^*(\beta^*) = \frac{1}{2\gamma}\|\beta^*\|_2^2$
    \item $\omega(\beta) = I\left\{\beta \leq \mathbf{1}\right\} \; \Rightarrow \; \omega^*(\beta^*) = \mathbf{1}^\top \beta^* + I(\beta^* \in \mathbb{R}_+^d).$
\end{enumerate}
Plugging these definitions into Theorem \ref{thm:primal_dual_} and removing constants yields the dual problem.
\begin{equation}
    \underset{\alpha \in \mathbb{R}_+^n}{\text{maximize }} -\frac{1}{2}\left\|\tau \mathbf{1} - \alpha\right\|_2^2 - \underset{v \in \mathbb{R}_-^d, u \in \mathbb{R}_+^d}{\text{min }}\frac{1}{2\gamma}\left\|X^\top \alpha - \lambda \mathbf{1} - v - u \right\|_2^2 + \mathbf{1}^\top u
\end{equation}
The effect of the $u,v$ variables is to clamp the entries of $X^\top \alpha - \lambda \mathbf{1}$ between $0$ and $1$. Each of the terms is differentiable at any $\alpha \in \mathbb{R}^n_+$ and the objective's gradient at $\alpha$ is given by
\[-\tau \mathbf{1} + \alpha + \frac{1}{\gamma}X\left[X^\top \alpha - \lambda \mathbf{1}\right]_{0}^1\]
where $[x]_a^b$ operates entrywise and sets the $i^{\text{th}}$ entry to $\text{min}(\text{max}(x_i, a), b)$.

In order to optimize this objective, note that the only non-differentiable aspect comes from the non-negativity constraint and is completely separable. This is a particularly simple non-differentiable problem which is amenable to active set algorithms where we compute a descent direction using the scheme in equation (\ref{eq:quasi_newton_}) on all $\alpha_i$ that are not clamped to $0$ as well as all $\alpha_i$ that are currently at $0$ but which would be set to a strictly positive value by the gradient. We use for $H = I + X_{\mathcal{A}}X^\top_{\mathcal{A}}$ where $\mathcal{A}$ contains all indices of $X^\top \alpha - \lambda \mathbf{1}$ that lie in $(0,1)$.
\subsubsection{Elastic-Net Regularized Logistic Regression}
The objective for Elastic-Net regularized logistic regression is
\begin{equation}
    -\sum_{i=1}^n \left[y_i x_i^\top \beta - \log\left(1 + e^{x_i^\top \beta}\right)\right] + \lambda \|\beta\|_1 + \frac{\tau}{2}\|\beta\|_2^2.
\end{equation}
The dual problem via Theorem \ref{thm:primal_dual_} and our table of conjugates is easily seen to be
\begin{equation*}
\begin{aligned}
    &\underset{\substack{\alpha \in \mathbb{R}^n \\ s \in [-1,1]^d}}{\text{max }}& -\sum_{i=1}^n\left[(\alpha_i + y_i)\log (\alpha_i + y_i) - (1 - \alpha_i - y_i)\log (1 - \alpha_i - y_i)\right] \\
    &&-\frac{1}{2\tau}\|X^\top \alpha - \lambda s\|_2^2\\
    &\text{subject to }& \mathbf{0} \preceq \alpha + y \preceq \mathbf{1}
\end{aligned}
\end{equation*}
with primal dual variables related by
\[\beta(\alpha) = \frac{1}{\tau}S(X^\top \alpha, \lambda).\]
Here $S(x, \lambda)$ is the soft-thresholding operator that sets each entry to
\[\text{sign}(x_i)\text{max}(|x_i| - \lambda,0).\]

The gradient at an interior point of the constraint set is given by
\[2 + \log \left[(\alpha+y)(1 - \alpha - y)\right] + \frac{1}{\tau}X\beta(\alpha),\]
where the $\log$ and vector products are elementwise. We use for the approximate Hessian in equation (\ref{eq:quasi_newton_})
\[\frac{1}{\alpha + y} - \frac{1}{1 - \alpha - y} + \frac{1}{\tau}X_{\mathcal{A}}X_{\mathcal{A}}^\top\]
where division is performed elementwise. The subscripts ${\mathcal{A}}$ indicate the columns of $X$ whose corresponding coefficients are not clamped to zero. Also, we handle the bound constraints on $\alpha$ by only computing a quasi-Newton step on the indices of $\alpha$ which are in the interior of the constraint set or which are at the boundary and whose gradient pushes them towards the interior. If, in the course of selecting a step size, a component of $\alpha$ runs into a boundary we fix that component but still allow the others to vary.
\subsubsection{Group Lasso Regularized Reduced Rank Regression}
We consider the objective 
\begin{equation}\label{eq:matrix_primal_}
    \frac{1}{2}\|Y - XW\|_F^2 + \rho \|XW\|_* + \lambda \|W\|_{2,1} + \frac{\tau}{2}\|W\|_F^2.
\end{equation}
In order to derive the dual of this problem, we compute the Fenchel conjugate of
\[f(P) = \frac{1}{2}\|Y - P\|_F^2 + \rho \|P\|_*.\]
Working with $\underset{P}{\text{min }} f(P) - \braket{A, P}$ instead, we have the minimization problem
\begin{equation}
    \underset{P}{\text{min }} \frac{1}{2}\|Y - P\|_F^2 + \rho \|P\|_* - \braket{A,P}.
\end{equation}
The nuclear norm may be rewritten as $\|P\|_* = \underset{V: \|V\|_2 \leq 1}{\text{max }} \braket{P, V}$, which leads to
\begin{equation}
    \underset{P}{\text{min }} \frac{1}{2}\|Y - P\|_F^2 - \braket{A,P} + \rho \underset{V: \|V\|_2 \leq 1}{\text{max }} \braket{P, V}.
\end{equation}
An application of Theorem \ref{thm:primal_dual_} shows that the conjugate is thus
\begin{equation}
    f^*(A) = \underset{V : \|V\|_2 \leq 1}{\text{min }}\;\frac{1}{2}\|Y + A - \rho V\|_F^2
\end{equation}
We are now ready to apply Theorem \ref{thm:primal_dual_} to equation (\ref{eq:matrix_primal_}). The overall dual problem is then
\begin{equation}\label{eq:matrix_dual_}
    -\underset{V : \|V\|_2 \leq 1}{\text{min }}\;\frac{1}{2}\|Y - A - \rho V\|_F^2 - \underset{S : \|S\|_{2,\infty} \leq 1}{\text{min }} \;\frac{1}{2\tau}\|X^\top A - \lambda S\|_F^2,
\end{equation}
where $\|M\|_{2,\infty}$ denotes the largest $\ell_2$ norm among the rows of $M$.

The primal variable is determined by
\[W(A) = \frac{1}{\tau}\left(X^\top A - \lambda S\right).\]
Observe that $\lambda S = DX^\top A$ where $D$ is diagonal with
\[D_{kk} = \begin{cases} 1 & \text{if } \|A^\top x_k\|_2 \leq \lambda \\
\frac{\lambda}{\|A^\top x_k\|_2} & \text{otherwise.}
\end{cases}.\]
Here $x_k$ denotes column $k$ of $X$.
This simplifies the primal variable calculation to
\begin{equation}
    W(A) = \frac{1}{\tau}\left(I - D\right)X^\top A.
\end{equation}
\subsubsection{Matrix Ranks}
This form of the primal variable also shows that the rank of $W(A)$ and $A$ is determined by the rank of $P$. For simplicity's sake we only discuss the cases when the matrix $X(I-D)$ has full rank when its zero columns are ignored. In particular, observe that the prediction on the training set is
\[P=XW(A) = \frac{1}{\tau}X(I-D)X^\top A.\]
If matrix $X(I-D)X^\top$ has full rank then $\text{rank}[P] = \text{rank}[W(A)] = \text{rank}[A]$. This first case can only occur when there are at least as many predictors in the model as there are training examples.

For the case when there are fewer predictors than training examples we let $\hat{X}$ be the submatrix of $X$ with nonzero columns in $X(I-D)$ and correspondingly define $\hat{D}$. We now assume that $\hat{X}$ has full rank, and we let $W^*,P^*,A^*$ be an optimal triplet pertaining to the primal, prediction, and dual variables. Suppose further that $A^* = BU + B_{\perp}U_{\perp}$ where $B$ is a basis for the range of ${X}(I-{D}){X}^\top$ and $B_{\perp}$ is a basis for its null space. Observe that
\[P^* = \frac{1}{\tau}X(I-D)X^\top A=\frac{1}{\tau}\hat{X}(I-\hat{D})\hat{X}^\top A=\frac{1}{\tau}\hat{X}(I-\hat{D})\hat{X}^\top BU.\]
For the primal coefficients, since $\hat{X}$ has full rank its columns are linearly independent. This implies that
\[\hat{X}(I-\hat{D})\hat{X}^\top u = \mathbf{0} \Leftrightarrow (I-\hat{D})\hat{X}^\top u = \mathbf{0}\]
since $I-\hat{D}$ has strictly positive entries on its main diagonal. It follows that 
\[(I-\hat{D})\hat{X}^\top B_\perp = \mathbf{0},\]
and hence
\[W^* = \frac{1}{\tau}(I-D)X^\top BU.\]
Since $P^*, W^*$ are primal optimal, we may equally well use $BU$ for our dual variable, and this must have the same rank as $P^*$.

\subsubsection{Gradient and Approximate Hessian}
In order to compute the gradient of
\[\underset{V : \|V\|_2 \leq 1}{\text{min }}\;\frac{1}{2}\|Y - A - \rho V\|_F^2\]
with respect to $A$, let $Y-A = U\Sigma Q^\top$ be the economical SVD of $Y-A$. The optimal $V$ is then given by
\[\rho V = U [\Sigma]^\rho Q^\top\]
where $[\Sigma]^\rho$ simply clips the singular values greater than $\rho$ to be $\rho$. The gradient of the dual loss is then given by
\[-\left(Y - A - U [\Sigma]^\rho Q^\top\right) = -U\left(\Sigma - [\Sigma]^\rho\right)Q^\top = -\sum_{k : \Sigma_{kk} > \rho} (\Sigma_{kk} - \rho) U_k Q_k^\top\]
where $U_k, Q_k$ are the $k^{\text{th}}$ columns of $U,Q$, respectively. We thus see that this gradient is low rank when $\rho$ is sufficiently large. 

The overall objective gradient is simply
\[-\sum_{k : \Sigma_{kk} > \rho} (\Sigma_{kk} - \rho) U_k Q_k^\top + \frac{1}{\tau}XW(A).\]
We use a quasi-Newton scheme as in equation (\ref{eq:quasi_newton_}), where our approximate Hessian is given by
\[H = I + \frac{1}{\tau}X\left(I-D\right)X^\top\]
with $D$ calculated from $W(A)$.

\bibliographystyle{unsrt}
\bibliography{mybib}

\end{document}